%% file: main.tex
\documentclass[10pt,twocolumn,letterpaper]{article}

\usepackage{cvpr}
\usepackage{times}
\usepackage{epsfig}
\usepackage{graphicx}
\usepackage{amsmath}
\usepackage{amssymb}
\usepackage{microtype}
\usepackage[linesnumbered,ruled]{algorithm2e}

\usepackage[pagebackref=true,breaklinks=true,letterpaper=true,colorlinks,bookmarks=false]{hyperref}
\usepackage{siunitx}
\usepackage[skip=1ex]{caption}

\usepackage[pagebackref=true,breaklinks=true,letterpaper=true,colorlinks,bookmarks=false]{hyperref}

\usepackage{bbm}
\usepackage{amsthm}
\usepackage{mathtools}
\usepackage{booktabs}
\usepackage[dvipsnames]{xcolor}

\newtheorem{theorem}{Theorem}
\newtheorem{proposition}[theorem]{Proposition}
\newcommand{\norm}[1]{\left\lVert#1\right\rVert}

\cvprfinalcopy %

\def\cvprPaperID{1105} %

\ifcvprfinal\fi

\input{defs}

\begin{document}

\title{Depth Sensing Beyond LiDAR Range }

\author{Kai Zhang$^1$
\and
Jiaxin Xie$^2$
\and 
Noah Snavely$^1$
\and
Qifeng Chen$^2$
\and
{
\centering
$^1$Cornell Tech, Cornell University \qquad $^2$HKUST
}
}

\maketitle

\begin{abstract}
\input{00_abstract.tex}
\end{abstract}

\input{01_intro.tex}
\input{03_problem.tex}

\input{04_method.tex}
\input{05_experiment.tex}

\input{06_conclusion.tex}

{\small
\bibliographystyle{ieee_fullname}
\bibliography{egbib}
}

\end{document}

%% file: defs.tex
\newenvironment{packed_item}{
\begin{itemize}
  \setlength{\itemsep}{1pt}
  \setlength{\parskip}{2pt}
  \setlength{\parsep}{0pt}
}{\end{itemize}}

%% file: 00_abstract.tex
Depth sensing is a critical component of autonomous driving technologies, but today's LiDAR- or stereo camera--based solutions have limited range. We seek to increase the maximum range of self-driving vehicles' depth perception modules for the sake of better safety. To that end, we propose a novel three-camera system that utilizes small field of view cameras. Our system, along with our novel algorithm for computing metric depth, does not require full pre-calibration and can output dense depth maps with practically acceptable accuracy for scenes and objects at long distances not well covered by most commercial LiDARs.

%% file: 01_intro.tex
\section{Introduction}

Depth perception is crucial in autonomous driving for obstacle avoidance, route planning, etc. Existing depth-sensing solutions typically rely on LiDAR, stereo cameras, or time-of-flight sensors~
\cite{waymo_open_dataset,caesar2019nuscenes,Geiger2012CVPR}. To the best of our knowledge, these systems have significant limits on their maximum ranges. For instance, the recent
Waymo~\cite{waymo_open_dataset} and nuScenes~\cite{caesar2019nuscenes} self-driving car datasets both feature LiDAR ranges of $\sim$80 meters, while the stereo cameras in KITTI~\cite{Geiger2012CVPR} cannot see distant objects in detail because of their wide field of view (FOV) (about $80^{\circ}$).\footnote{The cameras in the Waymo dataset have a 50-degree horizontal FOV.} It takes just 3 seconds for a vehicle to travel 80 meters at a speed of 60 miles/h, which is too short a time window in unforeseen emergency situations. While some high-end LiDARs claim to reach a maximum of 300 meters, e.g., Velodyne's Alpha Puck$^\text{TM}$~\cite{alpha_puck},
these are not only expensive but also produce very sparse point clouds for distant objects due to power and cost constraints. 
Such limited range can become a critical issue when a self-driving vehicle is a heavily weighted truck, or moving at high speed. The earlier an autonomous vehicle perceives the depth of obstacles on its driving route, the safer the technology is, as an early defensive response can be made in case of emergency.

Hence, 
there is a 
need for dense, accurate depth perception beyond the LiDAR range. In this work, we seek longer-range \textit{dense} depth sensing beyond the 200 meter range, 
which is not well covered by most existing commercial LiDARs for autonomous driving. To that end, we propose a cost-effective solution that utilizes three cameras with small fields of view. Equipped with telephoto lenses, these cameras can perceive faraway scenes or objects.~\footnote{Note that our system is not aimed to replace existing short-range LiDAR, but instead to complement it in a cost-effective way because long-range LiDAR sensors are expensive, power-inefficient and only capture sparse depth measurements for distant objects.} Our novel three-camera setup can resolve geometric ambiguities that arise in stereo systems based on only two small-FOV cameras. For small-FOV stereo cameras, such ambiguities are caused by (1) a small baseline/depth ratio, (2) difficulty in calibrating small-FOV cameras, and (3) maintaining the calibration during usage. Surprisingly, we can solve these problems by adding a specific third camera, without requiring a fully accurate calibration of camera parameters, by using a novel depth disambiguation algorithm. 
Our proposed three-camera system, along with our depth estimation algorithm, can produce dense depth maps without the need to fully pre-calibrate camera intrinsics and extrinsics. Moreover, it is robust to small vibrations in camera orientations that are inevitable for cameras attached to moving vehicles. We demonstrate the effectiveness of our approach with both synthetic and real-world data. Experiments show that our method can achieve a 3\% relative error at a distance of 300 meters in terms of depth estimation accuracy.

In summary, our contributions are three-fold. First, to our knowledge, our approach is the first to address the problem of \textit{dense} depth map acquisition at a range beyond that of most LiDARs in the domain of autonomous driving. Second, we propose a novel camera setup and depth estimation algorithm that requires only partial camera calibration. Third, we validate the effectiveness of our long-range depth-sensing system on both synthetic and real-world data.

%% file: 03_problem.tex
\section{Problem setup and related work} \label{sec:problem}

In this section, we formulate our problem setup, review prior work, and analyze the applicability of relevant existing algorithms to our problem.

For depth estimation at a distance over 200 meters, one seemingly straightforward solution is to construct a stereo camera system with two small-FOV cameras and attach it to the vehicle. However, there are several challenges with such a setup that distinguish it from typical stereo camera setups:
\begin{packed_item}
\item Because the baseline is restricted by the vehicle width (e.g., 2 meters), the baseline/depth ratio is very small in our problem setup, leading to a narrow triangulation angle for estimating 3D points from image correspondences. Hence the geometric setting of this problem is particularly ill-conditioned.
\item Unlike cameras with standard FOVs, small-FOV cameras are near-orthographic when a scene's depth variation is much smaller than its average depth. The absence of strong perspective effects can lead to problems when using standard checkerboard-based calibration of intrinsic and extrinsic stereo camera parameters~\cite{zhang2000flexible}. 
\item The reduced FOV increases the system's sensitivity to vibrations. A small perturbation in orientation will lead to a noticeable change in image content. Such vibrations are difficult to avoid in real-world moving-vehicle scenarios, even if the stereo camera is rigidly mounted. 
\end{packed_item}
A pratical solution to long-range depth estimation with small-FOV cameras must address these challenges. Note that our problem setup also shares similarities with those explored in areas including structure from motion (SfM), structure from small motion (SfSM), and uncalibrated stereo rectification and calibration of cameras with telephoto lenses.

\medskip
\noindent\textbf{SfM.} SfM algorithms aim to automatically recover camera poses from image collections~\cite{snavely2006photo,snavely2008modeling,agarwal2011building, schonberger2016structure}. The minimal case is two-view SfM, which is also an important component of multi-view SfM~\cite{beder2006determining}. 
Two-view SfM often works through the decomposition of an essential matrix obtained from intrinsically-calibrated images~\cite{longuet1981computer}, followed by bundle adjustment~\cite{triggs1999bundle}. However, essential matrix estimation is challenging in our case due to its ill-conditioned geometric setup. Another key issue is the \emph{bas-relief ambiguity}~\cite{szeliski1996shape,belhumeur1999bas} present in SfM when the baseline/depth ratio is small and the camera is near-orthographic. The bas-relief ambiguity can cause unwanted distortions of the reconstructed scene, leading to large depth estimation errors for distant objects.

\medskip
\noindent\textbf{SfSM.} Structure from small motion refers to the SfM problem under small camera motion. Previous work~\cite{Yu14,im2015high,ha2016high,Im2019Accurate3R} reconstructs scene geometry from video clips with accidental motion 
caused by handshake. These methods exploit multi-view redundancies in video clips to overcome the high depth uncertainty arising from the small baseline/depth ratio. 
However, SfSM requires a video clip as input for the sake of abundant redundant observations, which is not suitable for autonomous driving due to the real-time constraint, as well as the presence of moving objects such as vehicles and pedestrians. Moreover, Ha et al.~\cite{ha2017closed} observe that SfSM is also vulnerable to the bas-relief ambiguity, and try to reduce this ambiguity with a separate rotation estimation step. But they assume pre-calibration of camera intrinsics, which is not trivial for small-FOV cameras. In addition, the errors in their estimated rotations are relatively large, considering the tiny triangulation angles involved, yielding inaccurate depth estimates for distant scenes. 

\medskip
\noindent\textbf{Uncalibrated stereo rectification.} Stereo cameras are often pre-calibrated before deployment, allowing for online rectification using calibrated intrinsics and poses. The case when such calibration is unavailable has also been studied, e.g., by Loop and Zhang~\cite{loop1999computing}, and Hartley~\cite{hartley1999theory}. However, their methods assume that the fundamental matrix is known. As with two-view SfM, such methods will be brittle in the face of the ill-conditioned fundamental matrix estimation problem and the inherent bas-relief ambiguity.

\medskip
\noindent\textbf{Calibration of cameras with telephoto lenses.}  Huang et al.~\cite{huang2007calibrating} equip a pan-tilt camera with a telephoto lens to capture biometric features over a long range. They demonstrate the degeneracy of calibrating a long-focal-length camera with the 2D-2D correspondences from a checkerboard because the perspective effect is weak. They also show that 2D-3D correspondences are essential for calibrating such a camera. Our proposed approach, however, does not require full calibration of camera intrinsics and is more practically convenient. We only need to know the focal length, which can be read out from the lens's specification sheet.

%% file: 04_method.tex
\section{Method}\label{sec:method}
Our approach has two major components: the camera setup and the accompanying depth estimation algorithm. Details of both components are provided below.

\begin{figure}[t]
    \centering
    \includegraphics[width=0.650\columnwidth]{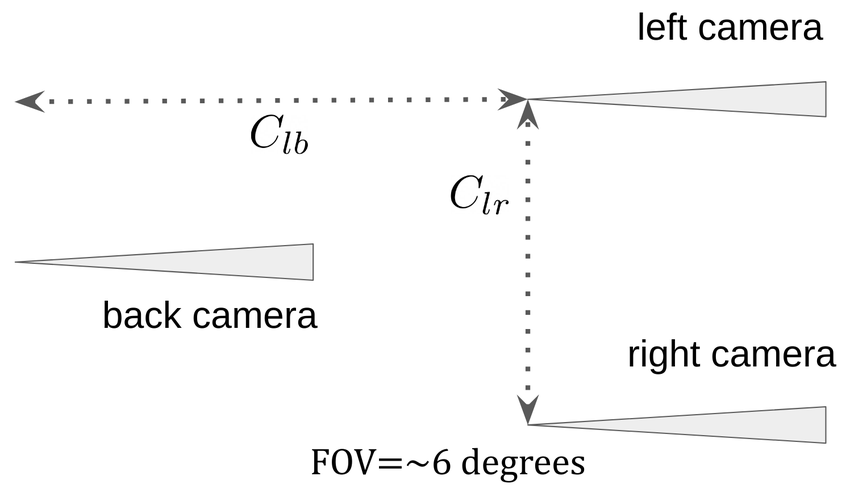}
    \caption{\small Top-down view of our proposed camera setup. The back camera can be positioned slightly higher than the left and right ones so that its view is not obscured. $C_{lr},C_{lb}\approx 2\text{m}$.}
    \label{fig:camera_setup}
\vspace{-2mm}
\end{figure}

Our camera setup requires three small-FOV cameras, placed according to Fig.~\ref{fig:camera_setup}. Two of the cameras form a left-right stereo pair, while a third is placed in the back of these two. The left and right cameras are mounted to a vehicle's front, with the back camera on the vehicle's tail. Each camera faces forward along the driving direction. 
We assume that the three cameras' focal lengths $f$ are the same and known. Furthermore, the distance between the optical centers of the left and right cameras, denoted $C_{lr}$, is assumed to be known, as well as the distance between left and back cameras $C_{lb}$ along the $z$-axis. No additional information is required. Like the baseline $C_{lr}$, the distance $C_{lb}$ should also be as large as possible to benefit subsequent processing. In practice, $C_{lr}$ and $C_{lb}$ are limited by vehicle size.

Our depth estimation algorithm takes the three images captured by our camera system as input and outputs a dense depth map for the left view. Our algorithm, detailed in Algo.~\ref{alg:depth_estimate}, is comprised of three modules: pseudo-rectification, stereo matching, and ambiguity removal. The pseudo-rectification step transforms the left and right images with affine warps so that they form a \textit{pseudo}-stereo pair. After pseudo-rectification, a standard stereo matching algorithm can be used to compute a disparity map. This disparity map, however, has an unknown constant offset as a result of pseudo-rectification. This ambiguity is resolved with the help of the back image, and the ambiguity-free disparity map is finally converted to a depth map.

\begin{algorithm}[t]
    \SetKwInOut{Input}{Input}
    \SetKwInOut{Output}{Output}

    \Input{left, right, back images; $f,C_{lr},C_{lb}$}
    \Output{depth map for left view}
    Pseudo-rectify left, right images \\
    Estimate disparity\\
    Remove ambiguity in the estimated disparity map \\
    Convert disparity to depth, and return
    \caption{Depth Estimation}
    \label{alg:depth_estimate}
\end{algorithm}

\medskip
\noindent\textbf{Step 1: Pseudo-rectification.} Standard stereo rectification utilizes the intrinsics and relative poses of two cameras to warp the left and right images with homographies such that the epipolar lines are aligned with image $x$-axis. Pure 2D methods assuming a known fundamental matrix have also been proposed, e.g., by Loop and Zhang~\cite{loop1999computing}. However, in our problem setup, neither the full intrinsics and relative pose are known accurately, nor can the fundamental matrix be reliably estimated in the case of a small baseline/depth ratio.  We propose a pseudo-rectification procedure based upon the observation stated in Prop.~\ref{prop:affine_homo}. Our algorithm approximately rectifies the two images with estimated affine transformations, and only depends on the left and right images.

\begin{proposition}
\label{prop:affine_homo}
When a small-FOV camera is rotated by a small amount, the homography warping the original image to the rotated view is approximately an affine transformation.
\end{proposition}

\begin{proof}
Let the 3D rotation be $\mathbf{R}=\mathbf{R}_z(\alpha)\mathbf{R}_y(\beta)\mathbf{R}_x(\gamma)$, where $\alpha,\beta,\gamma$ are Euler angles. Our proof strategy is to show that the three component-wise rotations all result in approximately affine transformations. First, the homography representing $\mathbf{R}_z(\alpha)$ is always affine. Now let's consider $\mathbf{R}_x(\gamma)$. Let $(x, y, z)$ be a 3D point in the camera coordinate frame, and denote $\frac{y}{z}=\tan\theta$, where $\theta$ is the angle between the $z$-axis and the vector $(0,y,z)$. Suppose $(x, y, z)$ becomes $(x', y',z')$ after $\mathbf{R}_x(\gamma)$ is applied. We then have
\begin{equation}
    x'=x,\frac{y'}{z'}=\tan (\theta -\gamma),z'=z \frac{\cos(\theta-\gamma)}{\cos \theta}.
\end{equation}
Because both the camera FOV and the rotation is small, both $\lvert \theta\rvert$ and $\lvert\theta-\gamma\rvert$ are roughly bounded by $\frac{FOV}{2}$ and hence small. We then have the approximations: 
\begin{equation}
\frac{x'}{z'}\approx \frac{x}{z}, \frac{y}{z}\approx \theta, \frac{y'}{z'}\approx \theta -\gamma .  
\end{equation}
We then project the 3D point into image space via $u=f\frac{x}{z}+c_x, v=f\frac{y}{z}+c_y$, where $f$ is the focal length, $(c_x, c_y)$ is the principal point, and $(u,v)$ are pixel coordinates. This gives us
\begin{equation}
u'\approx u, v'\approx v-f\gamma .
\end{equation}
This indicates that $\mathbf{R}_x(\gamma)$ approximately translates the image along the row axis. By similar logic, one can show that the homography representing $\mathbf{R}_y(\beta)$ is also approximately affine, which completes our proof.
\end{proof}

\begin{algorithm}[t]
    \SetKwInOut{Input}{Input}
    \SetKwInOut{Output}{Output}

    \Input{left and right images}
    \Output{pseudo-rectified left and right images}
    Initialize $\mathbf{H}^{(l)}=\mathbf{H}^{(r)}=[\mathbf{I}_{2\times 2},\mathbf{0}_{2\times 1}]$ \\

    Detect matches between left and right images \\

    \For{$t=1:T$}
    {
    Randomly sample $M$ matches\\
    Solve for candidate $\mathbf{H}^{(l)}_{21},\mathbf{H}^{(l)}_{22},\mathbf{H}^{(r)}_{21},\mathbf{H}^{(r)}_{22},\mathbf{H}^{(r)}_{23}$ \\
    If \# inliers increases, update the matrix entries
    } 
    
    Update $\mathbf{H}^{(l)}_{11},\mathbf{H}^{(l)}_{12},\mathbf{H}^{(r)}_{11},\mathbf{H}^{(r)}_{12}$ by enforcing the norm and determinant constraints \\

    Update $\mathbf{H}^{(r)}_{13}$ by imposing the disparity constraint\\
    
    Warp left and right images by $\mathbf{H}^{(l)}, \mathbf{H}^{(r)}$, respectively
    \caption{Pseudo-rectification}
    \label{alg:pseudo_rectify}
\end{algorithm}

Algorithmic details of our pseudo-rectification are specified in Algo.~\ref{alg:pseudo_rectify}. We use RANSAC~\cite{fischler1981random} to find a pair of rectifying affine transformations, $\mathbf{H}^{(l)},\mathbf{H}^{(r)}\in \mathcal{R}^{2\times 3}$, that map corresponding pixels to the same $y$-coordinates. This $y$-coordinate constraint only fixes some of the parameters in $\mathbf{H}^{(l)},\mathbf{H}^{(r)}$. To determine the rest, we need additional constraints. For instance, we choose $\mathbf{H}^{(l)}$ to be rigid, which preserves inter-pixel distances and is important for our disparity disambiguation. Other constraints are also imposed as needed.

In steps 1-2, we initialize two identity affine transformations, and detect sparse feature matches between the left and right images using SURF keypoints~\cite{bay2006surf}. Let the $N$ detected matches be $\{(\mathbf{x}^{(l)}_{i}, \mathbf{x}^{(r)}_{i}),i=1,\dots,N\}$, where %
$\mathbf{x}^{(l)}_{i},\mathbf{x}^{(r)}_{i}$
 are homogeneous pixel coordinates in the left and right views, respectively. In steps 3-7, we solve for $\mathbf{H}^{(l)}_{21},\mathbf{H}^{(l)}_{22},\mathbf{H}^{(r)}_{21},\mathbf{H}^{(r)}_{22},\mathbf{H}^{(r)}_{23}$ with RANSAC, by enforcing that corresponding pixels should have the same $y$-coordinate in the rectified views. At each RANSAC trial, a subset of $M$ matches
is randomly sampled; then we construct a homogeneous linear system of $M$ equations, with each sampled match resulting in one equation,
\begin{equation}
\langle\mathbf{H}^{(l)}_{2,1:3}, \mathbf{x}^{(l)}\rangle-\langle\mathbf{H}^{(r)}_{2,1:3}, \mathbf{x}^{(r)}\rangle=0
.\footnote{We use $\langle\cdot,\cdot \rangle$ to represent the inner product of two vectors. We also follow MATLAB notation of slicing matrices.  } \label{eq:equal_y}
\end{equation}
Additionally, because it is the difference between $\mathbf{H}_{23}^{(l)}$ and $\mathbf{H}_{23}^{(r)}$ that matters, rather than their absolute values, in Eq.~\ref{eq:equal_y}, we manually set $\mathbf{H}_{23}^{(l)}=0$. The SVD solution to the homogeneous linear system is also scaled such that $\mathbf{H}_{22}^{(l)}>0$ and $||\mathbf{H}_{2,1:2}^{(l)}||=1$.\footnote{$\norm{\cdot}$ denotes the $L_2$-norm of a vector unless otherwise noted.} In step 6, the number of inliers is defined as
\begin{equation}
\sum_{i=1}^N \mathbbm{1}\big\{\big\lvert \langle\mathbf{H}^{(l)}_{2,1:3}, \mathbf{x}^{(l)}_i\rangle-\langle\mathbf{H}^{(r)}_{2,1:3}, \mathbf{x}^{(r)}_i\rangle\big\rvert<\epsilon\big\},\label{eq:inlier}
\end{equation}
where $\mathbbm{1}\{\cdot \}$ is the indicator function, and $\epsilon$ is a threshold on the residual epipolar errors after rectification. In step 8, we solve for $\mathbf{H}^{(l)}_{11},\mathbf{H}^{(l)}_{12}$ by further imposing the norm constraint $||\mathbf{H}_{1,1:2}^{(l)}||=||\mathbf{H}_{2,1:2}^{(l)}||$ and the determinant constraint $\det(\mathbf{H}^{(l)}_{1:2,1:2})>0$. Similar constraints are also imposed on $\mathbf{H}^{(r)}$ to get $\mathbf{H}^{(r)}_{11},\mathbf{H}^{(r)}_{12}$. 
Finally, most existing stereo matching algorithms assume the disparity values to be all negative; thus in step 9, we set $\mathbf{H}^{(r)}_{13}$ to the 1-percentile of the set
\begin{equation}
\big\{\langle\mathbf{H}^{(l)}_{1,1:3}, \mathbf{x}^{(l)}_i\rangle-\langle\mathbf{H}^{(r)}_{1,1:3}, \mathbf{x}^{(r)}_i\rangle-\phi, i=1,\dots,N\big\},
\end{equation}
where $\phi=50$ pixels is a protective margin.
We then warp the left and right images with $\mathbf{H}^{(l)}$ and $\mathbf{H}^{(r)}$ respectively to obtain the pseudo-rectified stereo pair.

\medskip
\noindent\textbf{Step 2: Disparity estimation.} In our work, we adopt the state-of-the-art learning-based stereo matching method of Yang \etal using their provided pretrained model~\cite{yang2019hsm}. Other stereo matching algorithms can also be substituted into our pipeline.

\medskip
\noindent\textbf{Step 3: Ambiguity removal.} Because our pseudo-rectification method does not require accurate camera poses, the estimated disparity map is subject to an unknown global shift compared with that from \textit{true} stereo rectification. The unknown shift is physically linked to the unknown $y$-axis orientations of the left and right cameras (see the proof of Prop.~\ref{prop:affine_homo}), and mathematically reflected by the freedom to arbitrarily set $\mathbf{H}^{(l)}_{13}$ and $\mathbf{H}^{(r)}_{23}$ in our pseudo-rectification algorithm. This ambiguity prevents us from recovering absolute depth from disparity. To resolve it, one needs to know the ambiguity-free disparity value for at least one pixel of the rectified left view. This is equivalent to inferring one or more pixels' depths, because of the depth-to-disparity formula
\begin{equation}
d=f\cdot \frac{C_{lr}}{z}, \label{eq:disp_to_depth}
\end{equation}
where $d$ is a pixel's disparity and $z$ is its depth. Our ambiguity removal method utilizes the back view in our camera setup and is based on Prop.~\ref{prop:depth_esti} for inference of pixel depths. 

\begin{proposition}
\label{prop:depth_esti}
For two pixels in the left image with the same depth, if they are $m_{l}$ pixels apart, while their corresponding pixels in the back image are $m_{b}$ pixels apart, then the depth of these two pixels in the left camera's coordinate frame is 
\begin{equation}
z=\frac{C_{lb}}{\frac{m_{l}}{m_{b}}-1}.
\end{equation}
\end{proposition}

\begin{proof}
Denote the two same-depth pixels as $\mathbf{x}^{(l)}_1$ and $\mathbf{x}^{(l)}_2$ in the left image, and their corresponding 3D points as $\mathbf{X}^{(l)}_1$ and $\mathbf{X}^{(l)}_2$ in the camera coordinate frame. Then one can show,
\begin{equation}
    m_l=||\mathbf{x}^{(l)}_1-\mathbf{x}^{(l)}_2||=\frac{f}{z^{(l)}}\cdot ||\mathbf{X}^{(l)}_1-\mathbf{X}^{(l)}_2||,
\end{equation}
where $z^{(l)}$ is the common depth of $\mathbf{x}^{(l)}_1$ and $\mathbf{x}^{(l)}_2$. By similar logic and notation, for the back view, we have 
\begin{equation}
m_b=\frac{f}{z^{(b)}}\cdot ||\mathbf{X}^{(b)}_1-\mathbf{X}^{(b)}_2||.
\end{equation}
Because of our special camera setup, we have 
\begin{equation}
\mathbf{X}^{(l)}_1-\mathbf{X}^{(l)}_2 = \mathbf{X}^{(b)}_1-\mathbf{X}^{(b)}_2,z^{(b)} = z^{(l)}+C_{lb}.
\end{equation}
Hence,
\begin{equation}
\frac{m_l}{m_b}=\frac{z^{(l)}+C_{lb}}{z^{(l)}}.
\end{equation}
Rewriting the equation leads to  $z=z^{(l)}=\frac{C_{lb}}{\frac{m_{l}}{m_{b}}-1}$.
\end{proof}

\begin{algorithm}[t]
    \SetKwInOut{Input}{Input}
    \SetKwInOut{Output}{Output}

    \Input{rectified left image, back image, $f,C_{lr},C_{lb}$, and estimated disparity map}
    \Output{ambiguity-free disparity map}
    Detect matches between left and back images \\
    \For{$t=1:T$}
    {
    Randomly sample two matches\\
    \If {The two matches are far from each other in the left image, and have similar disparity values in the input disparity map}
    {
    Estimate the disparity offset, and cache it
    }
    }
    Shift the input disparity map by the median of all the cached disparity offset estimates
    \caption{Ambiguity Removal}
    \label{alg:ambiguity_remove}
\end{algorithm}

Details of our ambiguity removal algorithm can be found in Algo.~\ref{alg:ambiguity_remove}. We first detect sparse matches between the left and back images with SURF~\cite{bay2006surf}. Then, in steps 2-7, we estimate the unknown disparity offset for a number of times in order to reduce uncertainty of single measurement, each time with two matches randomly sampled from all the matches. Suppose the sampled two matches are $(\mathbf{x}^{(l)}_1,\mathbf{x}^{(b)}_1)$ and $(\mathbf{x}^{(l)}_2,\mathbf{x}^{(b)}_2)$ at each time, in which $\mathbf{x}^{(l)}_1=(u^{(l)}_1, v^{(l)}_1)$ is a pixel in the left view and similar rule applies to $\mathbf{x}^{(l)}_2, \mathbf{x}^{(b)}_1, \mathbf{x}^{(b)}_2$. Let the inter-pixel distances be denoted as $m_l=||\mathbf{x}^{(l)}_1-\mathbf{x}^{(l)}_2||$ in the left view, $m_b=||\mathbf{x}^{(b)}_1-\mathbf{x}^{(b)}_2||$ in the back view, and the values of the input disparity map at pixel locations $\mathbf{x}^{(l)}_1,\mathbf{x}^{(l)}_2$ are $d_1,d_2$, respectively. With the help of Prop.~\ref{prop:depth_esti} and Eq.~\ref{eq:disp_to_depth}, the offset $q$ resolving the ambiguity in the estimated disparity map can then be calculated as,
\begin{equation}
q=f\cdot\frac{C_{lr}}{C_{lb}}\cdot\left({\frac{m_{l}}{m_{b}}-1}\right)-\frac{d_1+d_2}{2}.\label{eq:ambiguity}
\end{equation}
To suppress uncertainties in Eq.~\ref{eq:ambiguity}, disparity offset estimation is only performed when the conditions (1) $m_l>m_b$, (2) $m_l>\delta$, and (3) $|d_1-d_2|<\eta$ are all satisfied, where $\delta$ and $\eta$ are preset thresholds. Condition (1) serves as a sanity check on whether the two sampled matches are physically valid. Condition (2) ensures that there is a sufficient difference between $m_l$ and $m_b$, while Condition (3) aims to guarantee that the two pixels have approximately equal depths. We take the median value of all the candidate disparity offset estimates, then shift the input disparity map to produce the ambiguity-free disparity map in step 8.

%% file: 05_experiment.tex
\begin{figure*}[t]
\begin{minipage}{0.5\textwidth}
\centering
\includegraphics[width=0.68\columnwidth, height=0.29\columnwidth]{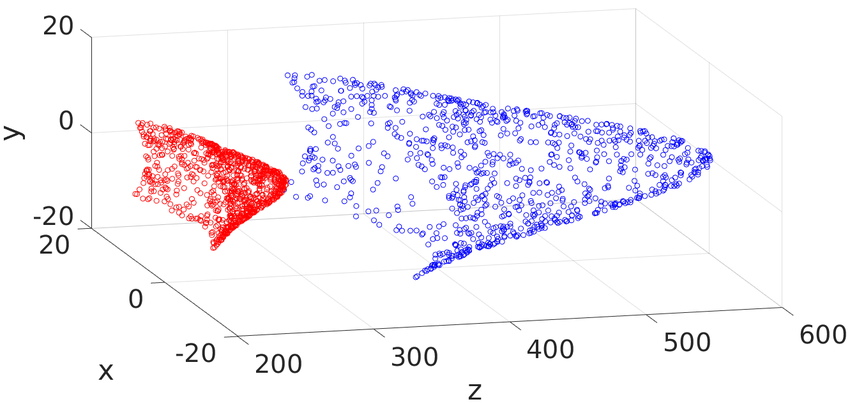}
  \caption{\small Ground-truth (blue) and the reconstructed (red) scene points. The unit for $x,y,z$ axes is meter. 
  }\label{fig:two_view_sfm}
\end{minipage} \quad
\begin{minipage}{0.5\textwidth}
\centering
\includegraphics[width=0.60\columnwidth, height=0.29\columnwidth]{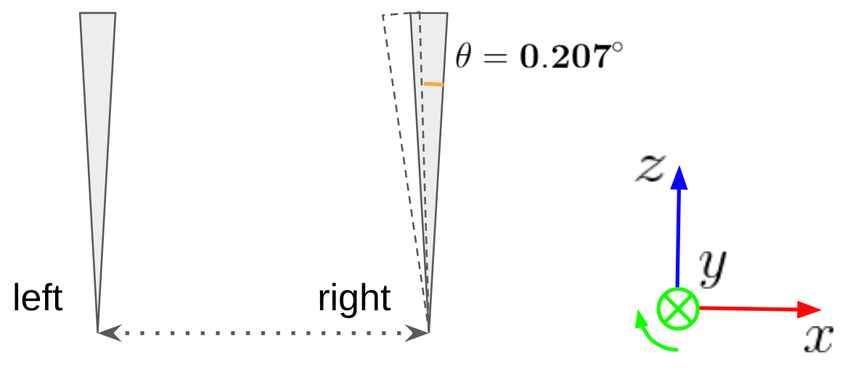}
\caption{\small Top-down view of ground-truth relative pose (solid) and the recovered one (dashed). $\theta$ is exaggerated for illustration.
}
\label{fig:two_view_sfm_pose}
\end{minipage}
\vspace{-2mm}
\end{figure*}

\section{Experiments}

In this section, we first illustrate what role the bas-relief ambiguity plays in our problem with a toy simulation example, and then show the effectiveness of our approach on both synthetic and real-world data. 

\subsection{Bas-relief ambiguity} \label{sec:simulate}
To give readers a more concrete understanding, we demonstrate the influence of the bas-relief ambiguity on our problem mentioned in Sec.~\ref{sec:problem} through a simulation. In our simulation, we first generate a Gaussian surface $z=a+b\cdot \exp{\left(-\frac{x^2+y^2}{2\sigma^2}\right)}$, in which $a=b=300$ and $\sigma=10$. We place the left camera at the origin and the right camera at $(2, 0, 0)$, and set both cameras' orientation to the identity. The image dimensions are set to 4608 $\times$ 3456, and the horizontal FOV to $6^\circ$, with centered principal points. 

We randomly sample 1,500 points from the Gaussian surface, and project them to the left and right images using the ground-truth poses. This yields 1,500 noise-free correspondences. To mimic real-world feature matching, we corrupt the projected pixel locations $(u, v)$ in the right image with random noise $(n_u, n_v)$ according to a 2D Gaussian distribution
$\mathcal{N}(0,\mathrm{diag}(1/\sqrt{2}, {1}/\sqrt{2}))$.
From these noisy matches, together with ground-truth camera intrinsics, we estimate the essential matrix and perform two-view SfM to recover the right camera's relative pose with respect to the left one. Since SfM has a scale ambiguity, we scale the recovered translation vector such that the estimated distance between the two camera centers is the same as in the ground-truth. Figures~\ref{fig:two_view_sfm} and \ref{fig:two_view_sfm_pose} show the reconstructed 3D points and the corresponding recovered relative pose in one of our multiple runs. 
Despite the translation and $x,z$-axis rotations are almost perfectly recovered,
the rotation about $y$-axis has a 0.207$^\circ$ error due to the bas-relief ambiguity, leading to severe distortions in the reconstruction. 

\subsection{Synthetic data}
\noindent\textbf{Setup.} We generate synthetic images, along with ground-truth depth maps, for a set of scenes.\footnote{The 3D models for rendering might not be in their real-world scale.} The horizontal camera FOV is set to $6^\circ$, with image dimensions 4608 $\times$ 3456. The corresponding focal length
is 43,963 pixels. For each scene, the left, right, and back images are rendered according to our camera setup in Fig.~\ref{fig:camera_setup}. To determine the cameras' positions, we first create a bounding box for the scene; then the left camera's pose is manually chosen such that the distance between its camera center and the bounding box centroid equals ${S}\big/\tan (\frac{\mathit{FOV}}{2})$, with $S$ being the bounding box's diagonal length. Both $C_{lr}$ and $C_{lb}$ are set to $1/150\cdot{S}\big/\tan (\frac{FOV}{2})$. Hence the baseline/depth ratio is as small as $\sim$$1/150$; in other words, the intersection angle for the corresponding rays in the left and right views is just  $\sim$$0.382^\circ$.  The setup is equivalent to that of sensing depth for objects $\sim$300m away with a 2m baseline. The relative orientations of the right and back cameras with respect to the left one are generated by randomly sampling their $x,y$ Euler angles from $[-1^\circ,1^\circ]$, and $z$ Euler angle from $[-5^\circ, 5^\circ]$.   

\smallskip
\noindent\textbf{Pseudo-rectification.} We test our proposed pseudo-rectification method on this synthetic data. We set the number of sampled matches $M$ at each RANSAC trial to 10, and the inlier epipolar error threshold $\epsilon$ to 2 pixels. In Fig.~\ref{fig:synthetic_pseudo_rectify}, we show an example for which both the true rectification with ground-truth poses and our purely image-based pseudo-rectification are performed. %
To facilitate visual inspection, we show two $160\times120$ crops of the pseudo-rectified views. Their locations are marked by the red boxes in the uncropped images. 
In addition to the rectification quality,  the horizontal disparity is also visible from the crops. 
We then process the pseudo-rectified stereo pair with a stereo matching method~\cite{yang2019hsm}. In Fig.~\ref{fig:ambiguity_dist}, we show the estimated disparity map, along with the ground-truth, to illustrate the existence of an unknown global shift ($\sim$250px).

\begin{figure}[t]
\centering
    \begin{tabular}{c@{\hspace{0.1em}}c@{\hspace{0.1em}}}
         \includegraphics[width=0.4\columnwidth]{{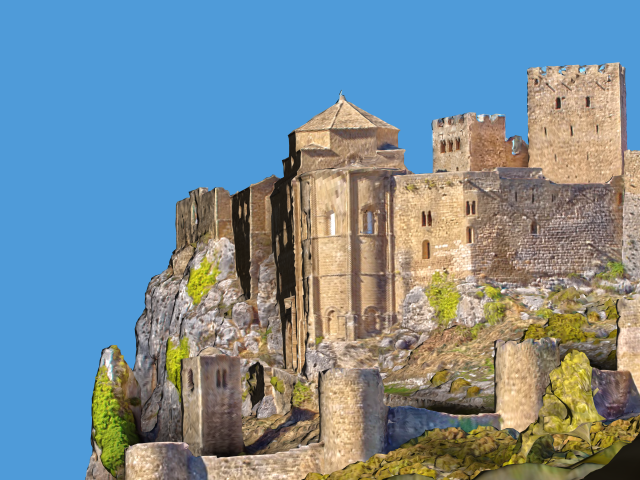}} &
         \includegraphics[width=0.4\columnwidth]{{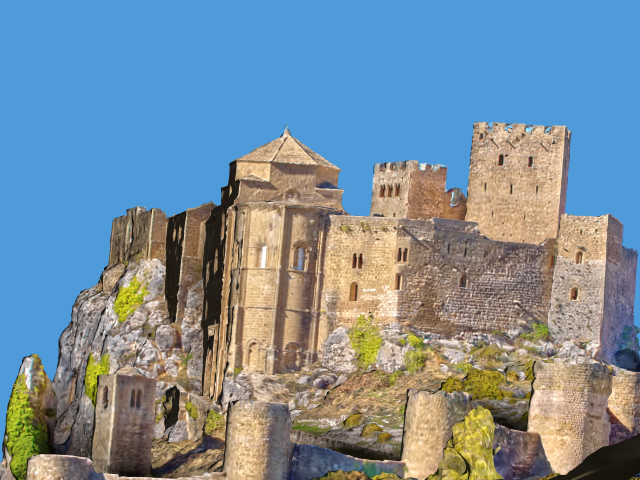}} \vspace{-0.4em}\\
        \multicolumn{2}{c}{\small Raw left and right views}\\ 
         \includegraphics[width=0.4\columnwidth]{{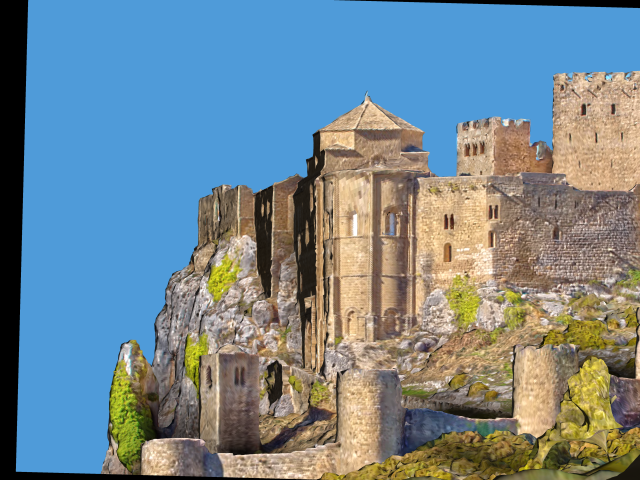}} &
         \includegraphics[width=0.4\columnwidth]{{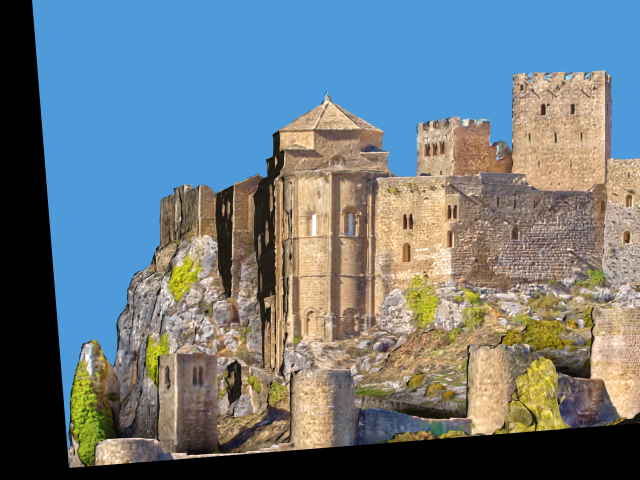}} \vspace{-0.4em}\\
        \multicolumn{2}{c}{\small True-rectified left and right views}\\ 
         \includegraphics[width=0.4\columnwidth]{{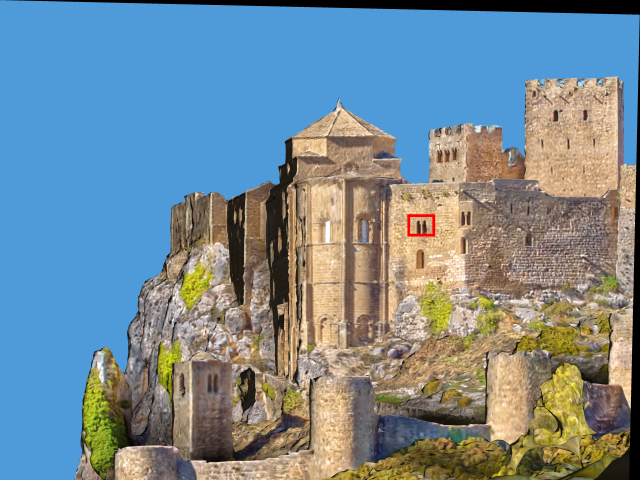}} &
         \includegraphics[width=0.4\columnwidth]{{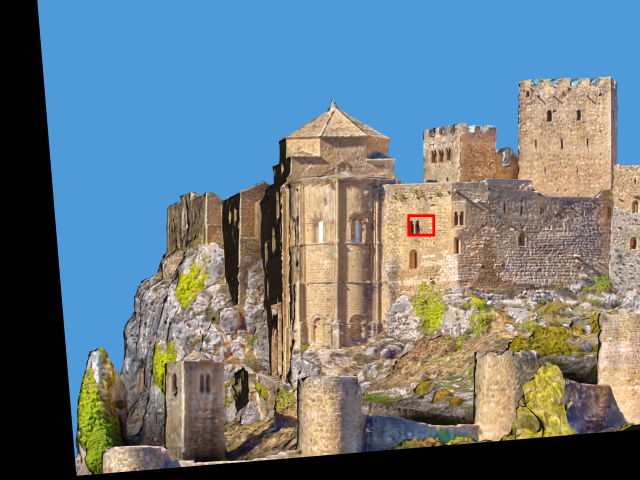}} \vspace{-0.2em}\\
         \includegraphics[width=0.4\columnwidth]{{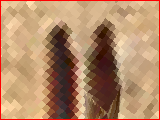}} &
         \includegraphics[width=0.4\columnwidth]{{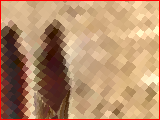}} \vspace{-0.4em} \\
        \multicolumn{2}{c}{\small Pseudo-rectified left and right views}\\ 
     \end{tabular} 
  \caption{\small Pseudo-rectification on synthetic images.
  }\label{fig:synthetic_pseudo_rectify}
\vspace{-2mm}
\end{figure}

\begin{figure}[t]
\centering
    \begin{tabular}{c@{\hspace{0.1em}}c@{\hspace{0.1em}}}
         \includegraphics[width=0.4\columnwidth]{{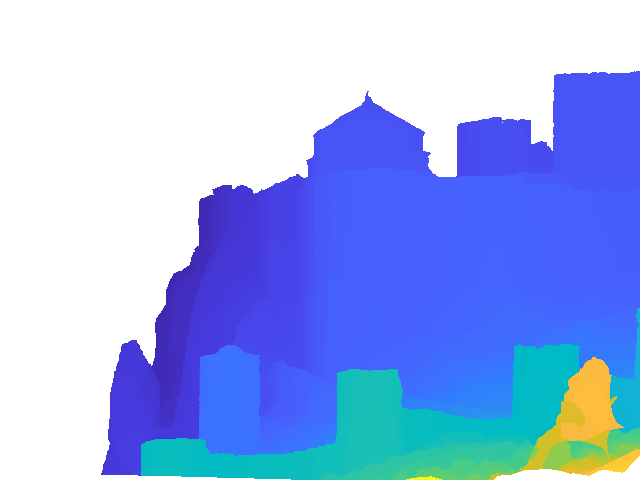}} &
         \includegraphics[width=0.4\columnwidth]{{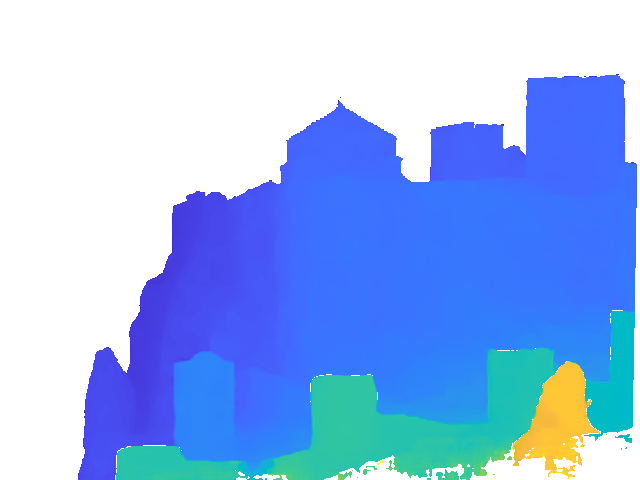}} \vspace{-1em}\\
         \includegraphics[width=0.4\columnwidth]{{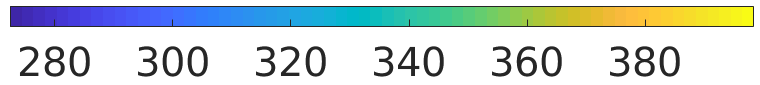}} &
         \includegraphics[width=0.4\columnwidth]{{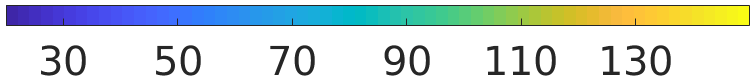}} \vspace{-0.4em}\\
        \multicolumn{1}{c}{\small Ground-truth disparity} & \multicolumn{1}{c}{\small Estimated disparity}\\ 
     \end{tabular} 

\smallskip

\centering
\includegraphics[width=0.75\columnwidth,height=0.36\columnwidth]{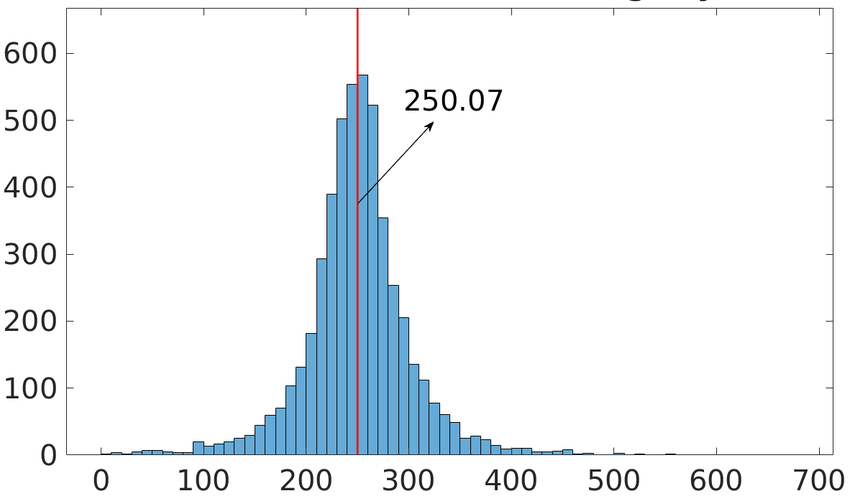} \\
\caption{\small Histogram of 5,000 cached disparity offset estimates for the example in Fig.~\ref{fig:synthetic_pseudo_rectify}. The red line marks the final value we take.
}\label{fig:ambiguity_dist}

\vspace{2mm}
\smallskip
\centering
\includegraphics[width=0.8\columnwidth]{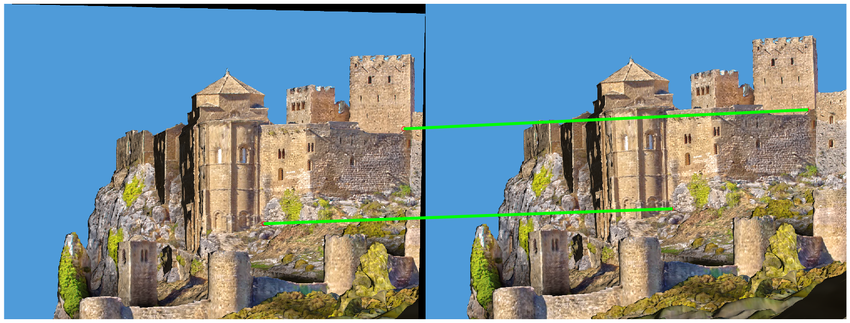}
  \caption{\small An example pair of matches used to resolve ambiguity for the example in Fig.~\ref{fig:synthetic_pseudo_rectify}. The two points are 1849.2px apart in the left image, and 1836.7px apart in the back one. Their original estimated disparities are 49.0px and 50.5px, respectively. According to Eq.~\ref{eq:ambiguity}, this yields a disparity offset estimate of 249.4px. 
}\label{fig:ex_match}
\vspace{-2mm}
\end{figure}

\noindent\textbf{Ambiguity removal.} Next, we evaluate our ambiguity removal algorithm. We set the inter-pixel distance threshold $\delta$ to 300 pixels, and the disparity difference threshold $\eta$ to 3 pixels in our ambiguity removal step. For the synthetic example in Fig.~\ref{fig:synthetic_pseudo_rectify}, Fig.~\ref{fig:ambiguity_dist} shows the histogram of all the cached disparity offset estimates produced by step 2-7 of Algo.~\ref{alg:ambiguity_remove}, while Fig.~\ref{fig:ex_match} visualizes an example pair of matches that meet our criterion in step 4. The final value taken in step 8 is marked by the red line in the histogram plot; we can see that it is aligned with the mode of the histogram, and also in agreement with the two disparity maps. We finally convert the ambiguity-free disparity map to a depth map via Eq.~\ref{eq:disp_to_depth}. The estimated depth map, compared with the ground-truth, are presented in the first row of 
Fig.~\ref{fig:synthetic_compare_alg}. One can see that our proposed method outputs a depth map with relative errors below 3\% at the majority (95.4\%) of pixel locations.  Another two synthetic examples can be seen in Fig.~\ref{fig:synthetic}.

\smallskip
\noindent \textbf{Loop and Zhang's rectification.} As a comparison, we replace our pseudo-rectification with Loop and Zhang's rectification scheme~\cite{loop1999computing}, while keeping the other parts of our pipeline unchanged. The fundamental matrix required by their approach is estimated with a RANSAC-based normalized 8-point algorithm~\cite{hartley2003multiple} from the same set of matches as that in our pseudo-rectification. Their results are shown in the second row of Fig.~\ref{fig:synthetic_compare_alg}.
The relative error map indicates that the depth map is strongly distorted when their method is used, for a similar reason to two-view SfM that we demonstrate in Sec.~\ref{sec:simulate}. In experiments, we also find that fundamental matrix estimation is quite unstable due to the tiny baseline/depth ratio, which causes Loop and Zhang's method to produce inconsistent results in different runs.

\begin{figure}[t]
\centering
    \begin{tabular}{c@{\hspace{0.02em}}c@{\hspace{0.02em}}c@{\hspace{0.02em}}}
            {\small Ground-truth depth} &
        {\small Estimated depth} & {\small Relative error (\%)} \vspace{-0.4em}\\
        \includegraphics[width=0.33\columnwidth]{{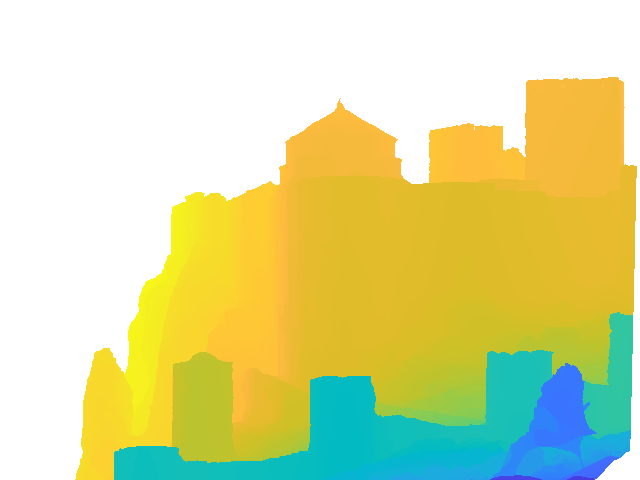}}  & 
        \includegraphics[width=0.33\columnwidth]{{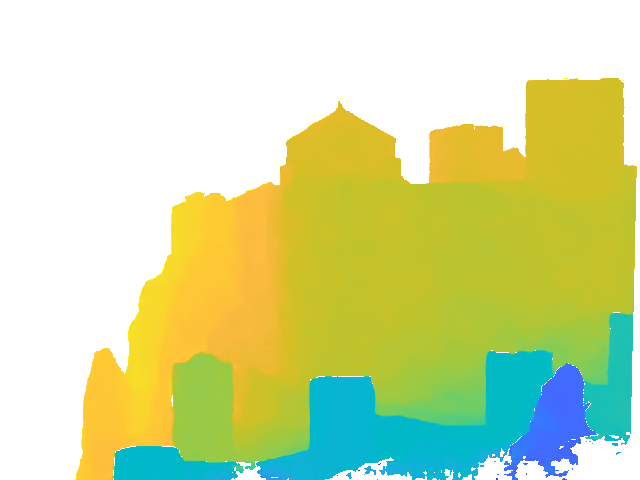}} & 
        \includegraphics[width=0.33\columnwidth]{{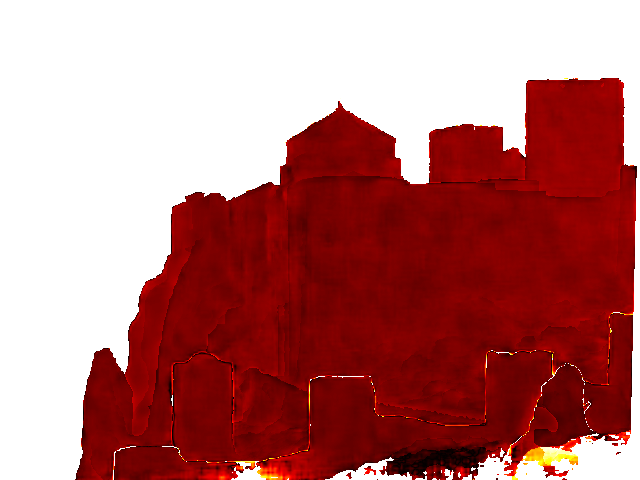}} \vspace{-0.3em}\\
        \includegraphics[width=0.33\columnwidth]{{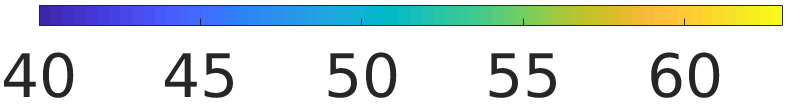}} & 
        \includegraphics[width=0.33\columnwidth]{{imgs/synthetic_loarre_castle/cbar_depth.png}} & 
        \includegraphics[width=0.33\columnwidth]{{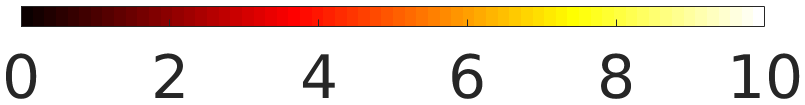}} 
        \vspace{-0.2em}\\
        \multicolumn{3}{c}{\small Our method}
        \vspace{-3mm}
        \\
        \includegraphics[width=0.33\columnwidth]{{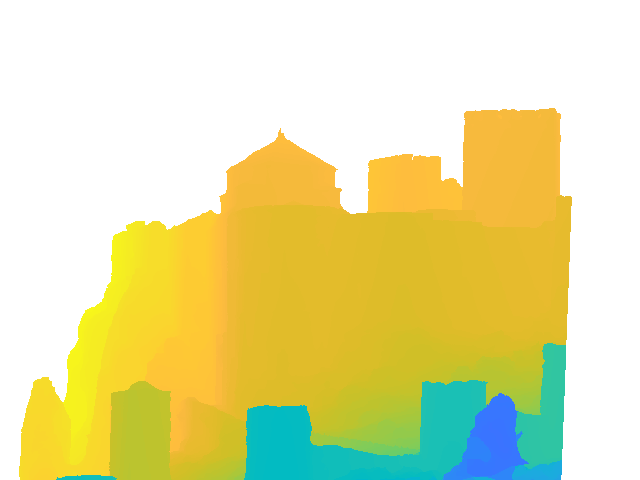}} &
        \includegraphics[width=0.33\columnwidth]{{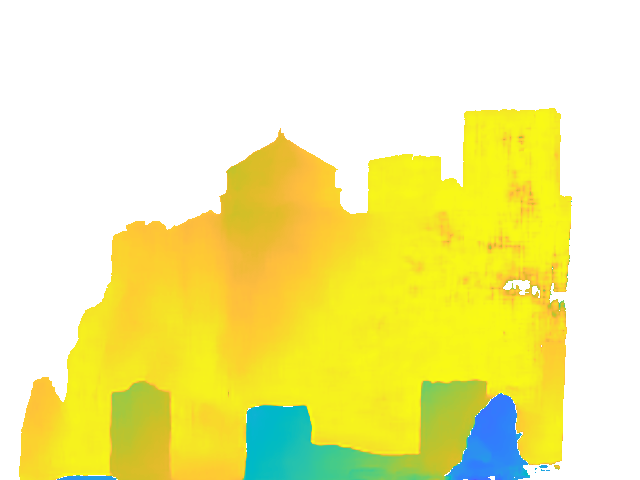}} &
        \includegraphics[width=0.33\columnwidth]{{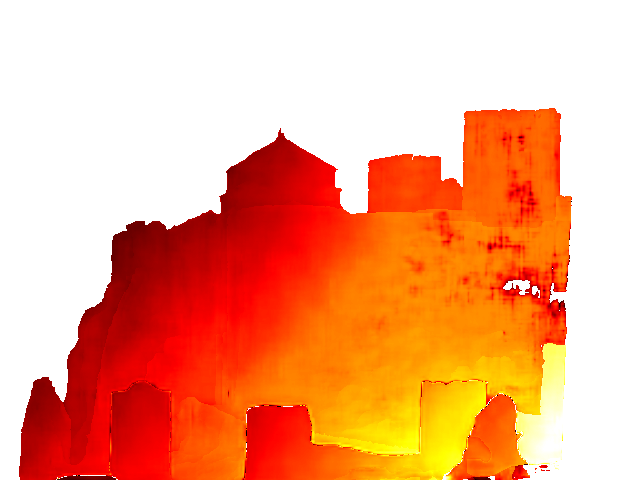}} \vspace{-0.3em}\\
        \includegraphics[width=0.33\columnwidth]{{imgs/synthetic_loarre_castle/cbar_depth.png}} &
        \includegraphics[width=0.33\columnwidth]{{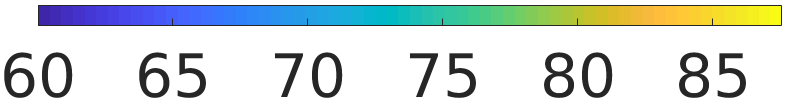}} &
        \includegraphics[width=0.33\columnwidth]{{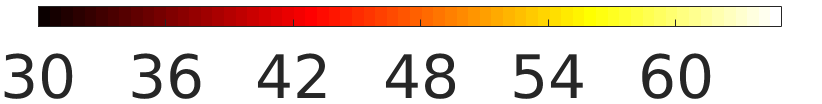}} 
        \vspace{-0.2em}\\
        \multicolumn{3}{c}{\small Replacing our pseudo-rectification with Loop and Zhang \cite{loop1999computing}} 
        \\
       \includegraphics[width=0.33\columnwidth]{{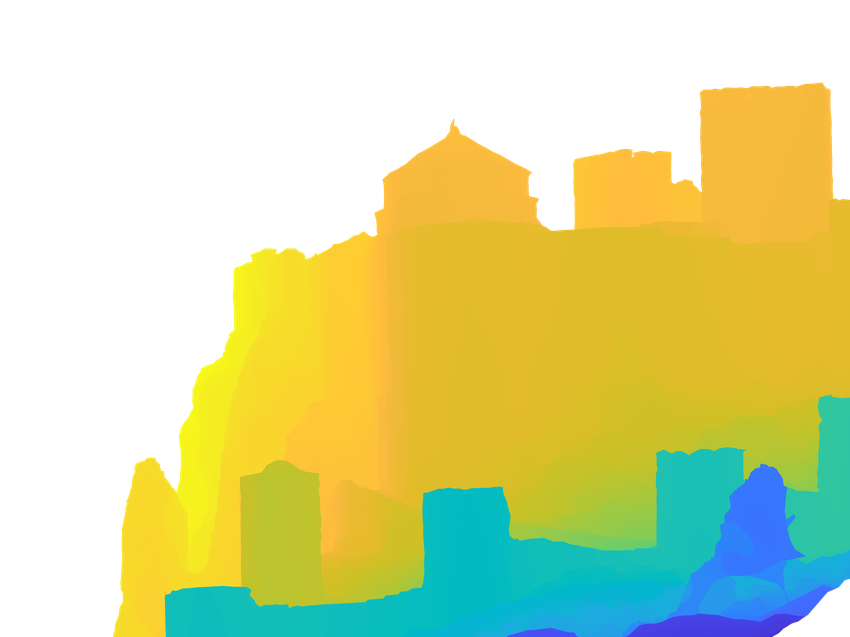}} & \includegraphics[width=0.33\columnwidth]{{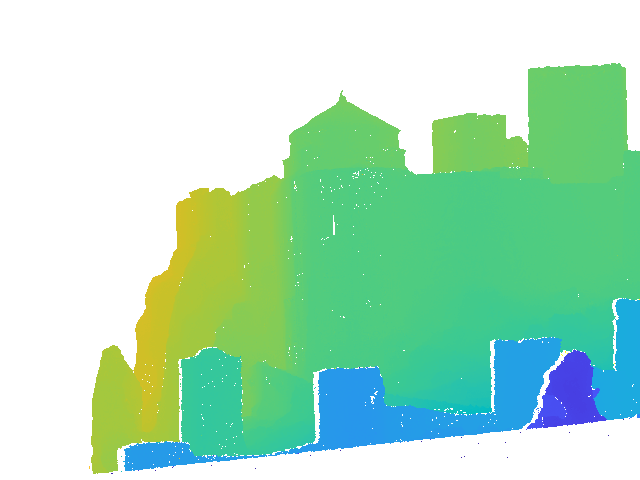}} & 
        \includegraphics[width=0.33\columnwidth]{{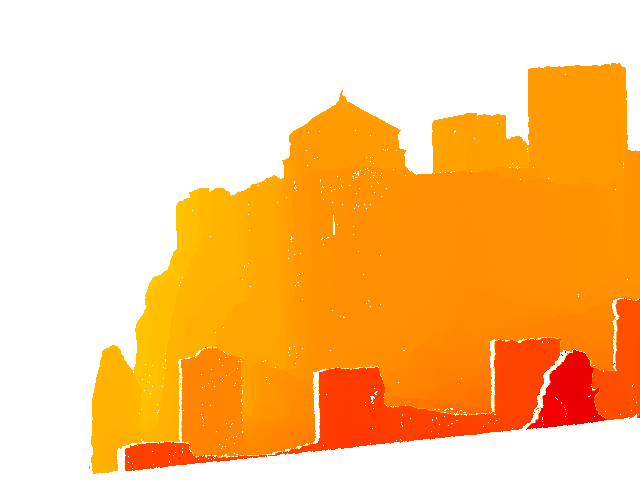}} \vspace{-0.3em}\\
        \includegraphics[width=0.33\columnwidth]{{imgs/synthetic_loarre_castle/cbar_depth.png}} & \includegraphics[width=0.33\columnwidth]{{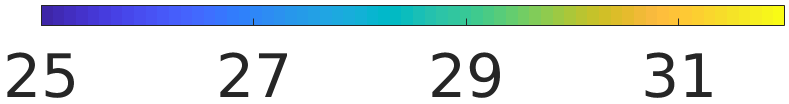}} & 
        \includegraphics[width=0.33\columnwidth]{{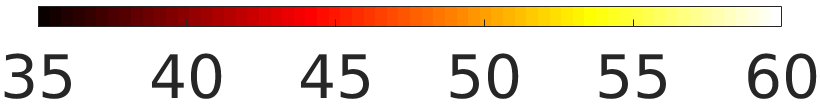}} 
        \vspace{-0.2em}\\
        \multicolumn{3}{c}{\small Multi-view SfM and MVS~\cite{schonberger2016structure,Schnberger2016PixelwiseVS}} 
        \\
     \end{tabular} 
  \vspace{-1mm}
  \caption{\small Comparison among different algorithms. For rectification-based methods, the ground-truth depth map has been warped to align with the rectified view. For SfM, we have used the full ground-truth intrinsic matrix. 
  }\label{fig:synthetic_compare_alg}
\vspace{-2mm}
\end{figure}

\smallskip
\noindent \textbf{SfM+MVS.} One might hypothesize that the back view can also help fix the bas-relief ambiguity in SfM. To test it, we feed
the three views into COLMAP~\cite{schonberger2016structure,Schnberger2016PixelwiseVS} to run multi-view SfM and MVS. We use the ground-truth camera intrinsics, and initialize the three cameras' orientations to the identity; the left, right, and back camera centers are initialized to their ground-truth locations.
The recovered pose and reconstruction are finally scaled such that the distance between the left and right camera centers is the same as in the ground-truth. Fig.~\ref{fig:synthetic_compare_alg} shows that even with the additional back view, SfM still suffers from the bas-relief ambiguity as in the case of stereo views. One of the key factors distinguishing SfM/SfSM from our approach is that, their bundle adjustment objective, i.e., average reprojection error, treats all the image-space observations equally, most of which are actually not informative for fixing the ambiguity, while our method only exploits a small carefully-chosen subset of all the observations, i.e., same-depth pixel pairs. Moreover, in practice, the principal points in camera intrinsics are unknown and can be tens of pixels away from the image center; this has no effect on our method, but can further hurt SfM/SfSM.

Tab.~\ref{tab:quantitative_compare}  quantitatively compares the aforementioned different algorithms on 40 synthetic scenes.  Unlike other methods, our approach is not affected by the bas-relief ambiguity and outputs much more accurate depth estimates. 

\begin{table}[t]
\centering
\begin{tabular}{@{\hspace{1mm}}l@{\hspace{3mm}}c@{\hspace{3mm}}c@{\hspace{3mm}}c@{\hspace{3mm}}c@{\hspace{1mm}}}
\toprule
    & Failure &$<$1\% & $<$2\% & $<$3\%  \\
    \midrule
    Ours & 0  & \textbf{45.3\%}& \textbf{80.1\%} & \textbf{96.9\%} \\
    Loop and Zhang \cite{loop1999computing} & 0 & 1.14\%& 2.73\%& 5.99\%\\
    SfM+MVS~\cite{schonberger2016structure,Schnberger2016PixelwiseVS} & 15 & 6.71\% &12.7\%& 19.1\%\\
\bottomrule
\end{tabular}
\caption{\small Quantitative results on 40 synthetic scenes for methods in Fig.~\ref{fig:synthetic_compare_alg}.   ``Failure'' means the number of scenes for which a method fails to output a depth map. The metric is the portion of pixels with relative depth error below certain threshold, i.e., 1\%, 2\%, 3\%, averaged over the successful scenes. 
}
\label{tab:quantitative_compare}
\end{table}

\begin{figure}[t]
\centering
    \begin{tabular}{c@{\hspace{0.1em}}c@{\hspace{0.1em}}}
         \includegraphics[width=0.35\columnwidth]{{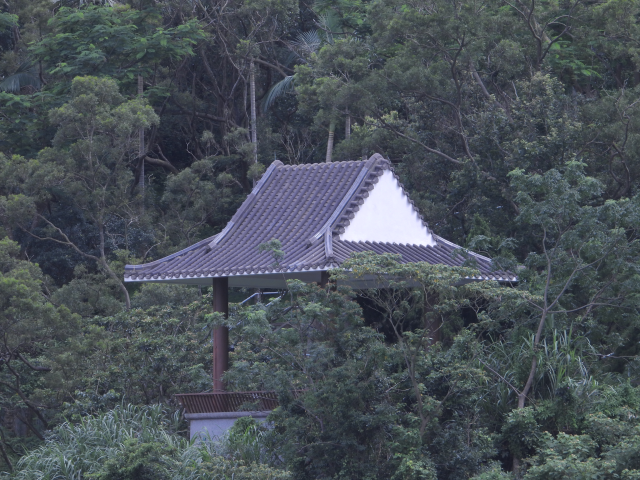}} &
         \includegraphics[width=0.35\columnwidth]{{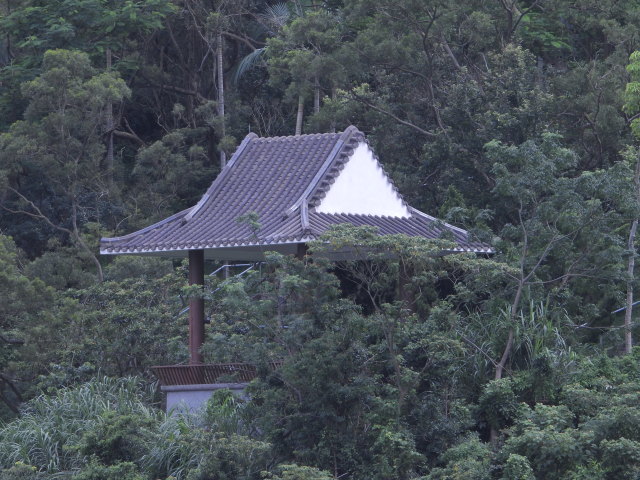}} \vspace{-0.4em}\\
        \multicolumn{2}{c}{\small Raw left and right views}\\ 

     \end{tabular} 
     
    \begin{tabular}{c@{\hspace{0.1em}}c@{\hspace{0.1em}}c@{\hspace{0.1em}}c@{\hspace{0.1em}}}
         \includegraphics[width=0.24\columnwidth]{{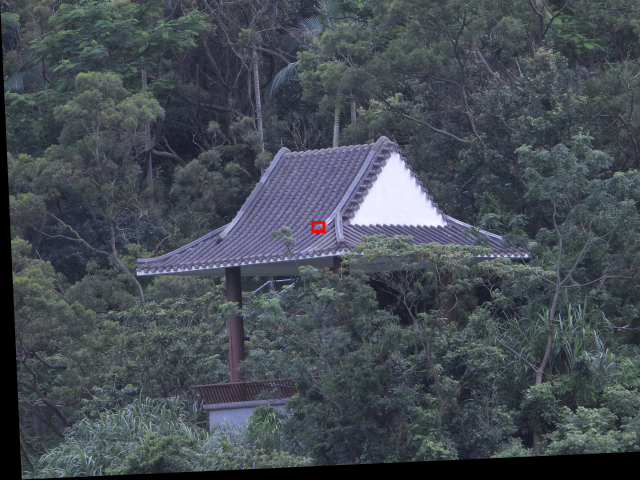}} &
         \includegraphics[width=0.24\columnwidth]{{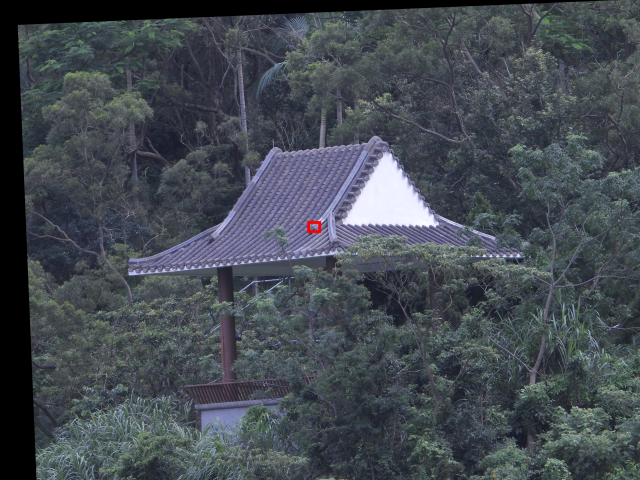}} &
         \includegraphics[width=0.24\columnwidth]{{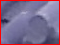}} &
         \includegraphics[width=0.24\columnwidth]{{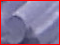}} \vspace{-0.4em}\\
        \multicolumn{4}{c}{\small Pseudo-rectified left and right views}\\
     \end{tabular} 
  \vspace{-1mm}
  \caption{\small Pseudo-rectification on real-world images.
  }\label{fig:real_pseudo_rectify}
\vspace{-2mm}
\end{figure}

\begin{figure*}[!htb]
\centering
    \begin{tabular}{c@{\hspace{0.5em}}c@{\hspace{0.5em}}c@{\hspace{0.5em}}c@{\hspace{0.5em}}c}
         {\small Pseudo-rectified left view} & {\small Ground-truth depth} & {\small Estimated depth} & {\small Relative error(\%)} \\
         
        \includegraphics[width=0.188\textwidth]{{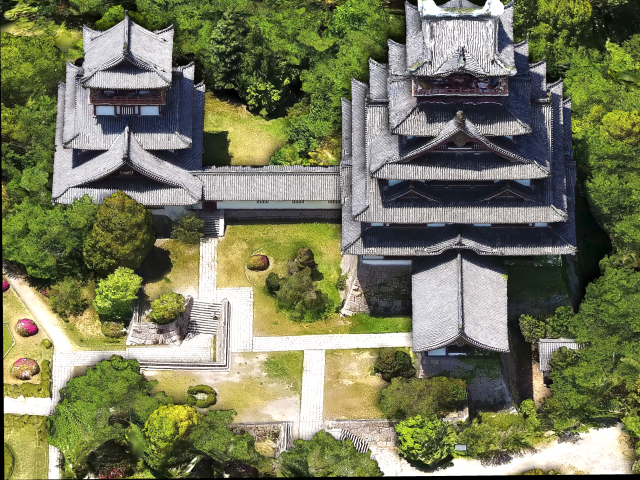}} & 
         \includegraphics[width=0.188\textwidth]{{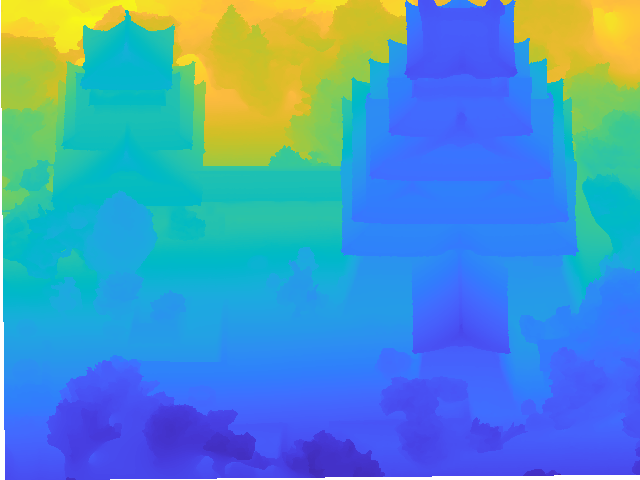}} & 
         \includegraphics[width=0.188\textwidth]{{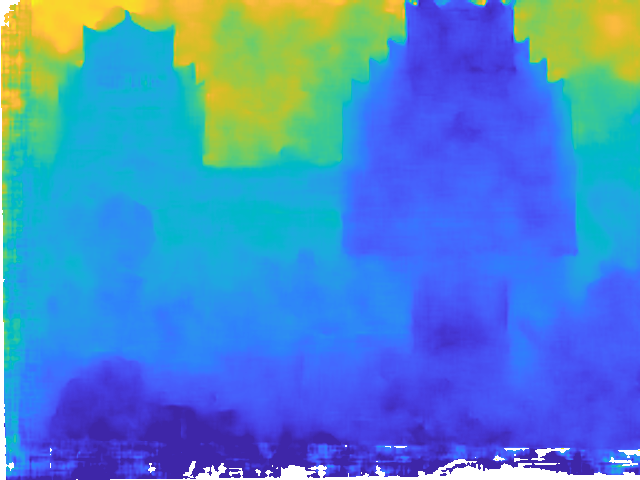}} & 
         \includegraphics[width=0.188\textwidth]{{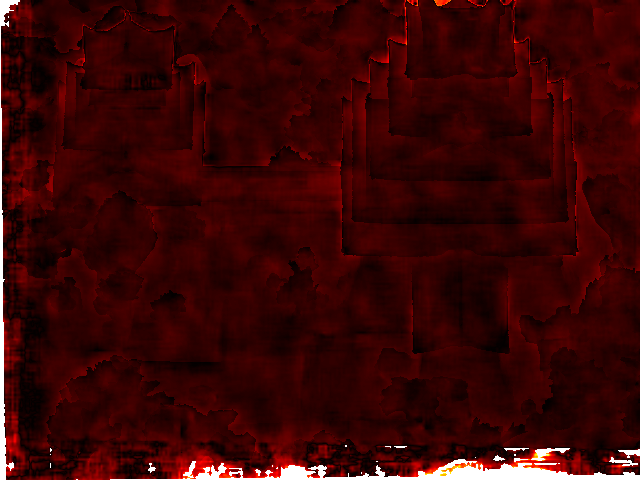}} \vspace{-0.4em}\\
         & 
         \includegraphics[width=0.188\textwidth]{{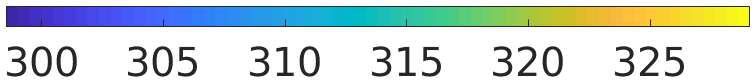}} & 
         \includegraphics[width=0.188\textwidth]{{imgs/synthetic_momoyam_castle/cbar_depth.png}} & 
         \includegraphics[width=0.188\textwidth]{{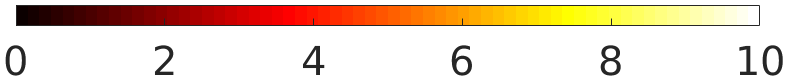}} \\
         
        \includegraphics[width=0.188\textwidth]{{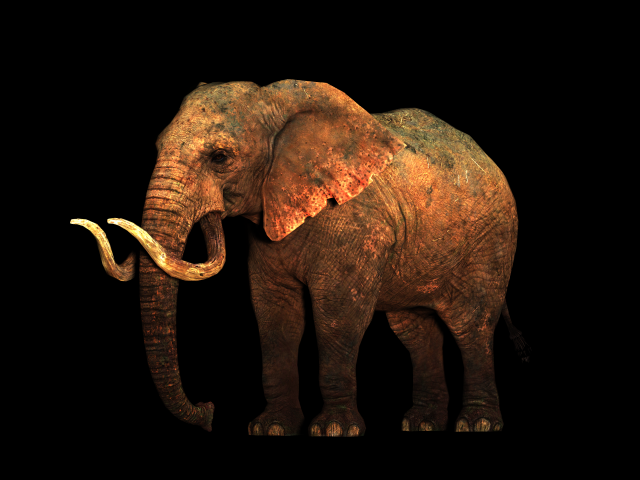}} & 
         \includegraphics[width=0.188\textwidth]{{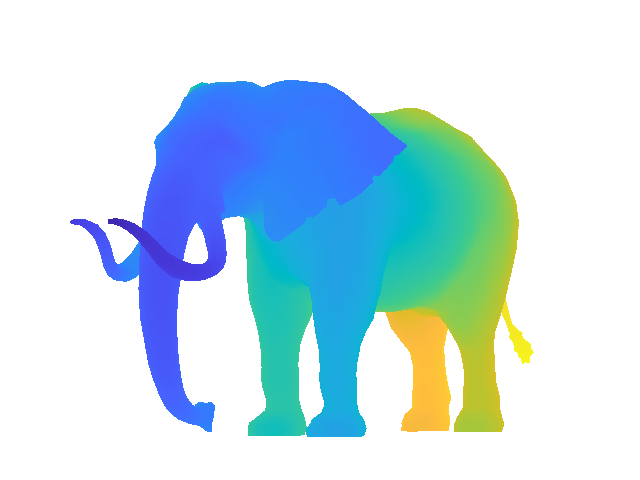}} & 
         \includegraphics[width=0.188\textwidth]{{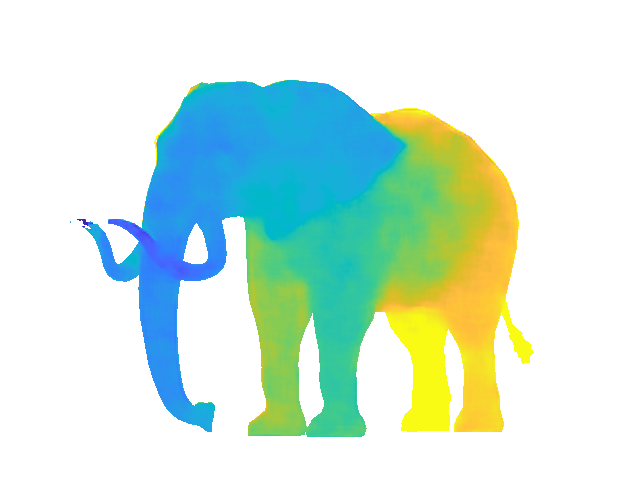}} & 
         \includegraphics[width=0.188\textwidth]{{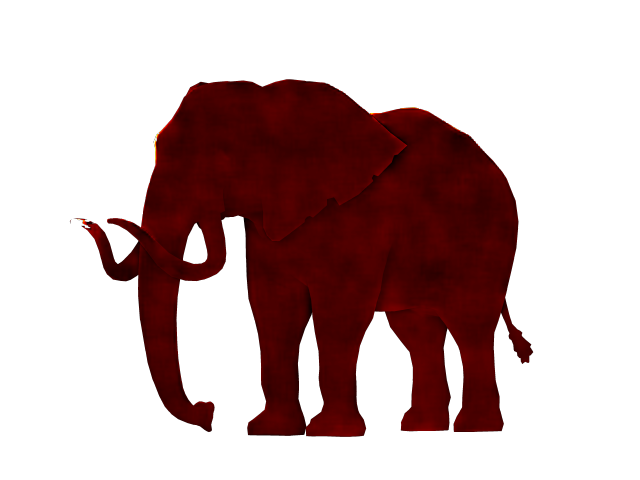}} \vspace{-0.4em}\\
         & 
         \includegraphics[width=0.188\textwidth]{{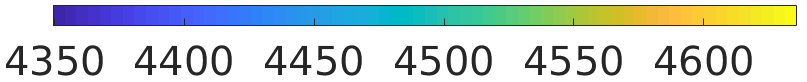}} & 
         \includegraphics[width=0.188\textwidth]{{imgs/synthetic_elephant/cbar_depth.png}} & 
         \includegraphics[width=0.188\textwidth]{{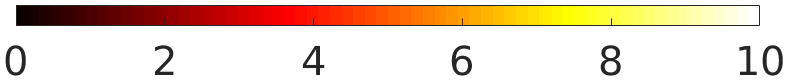}} \\
     \end{tabular} 
     \vspace{-1mm}
  \caption{\small Results on synthetic data. Portion of pixel locations with $<$3\% relative error  (top to bottom): 98.5\%, 99.7\%. 
  }\label{fig:synthetic}

\medskip

\centering
    \begin{tabular}{c@{\hspace{0.5em}}c@{\hspace{0.5em}}c@{\hspace{0.5em}}c@{\hspace{0.5em}}c}
         {\small Pseudo-rectified left view}  & {\small Estimated disparity}  & {\small Estimated disparity offset} & {\small Estimated depth (meters)} \\
         \includegraphics[width=0.188\textwidth]{{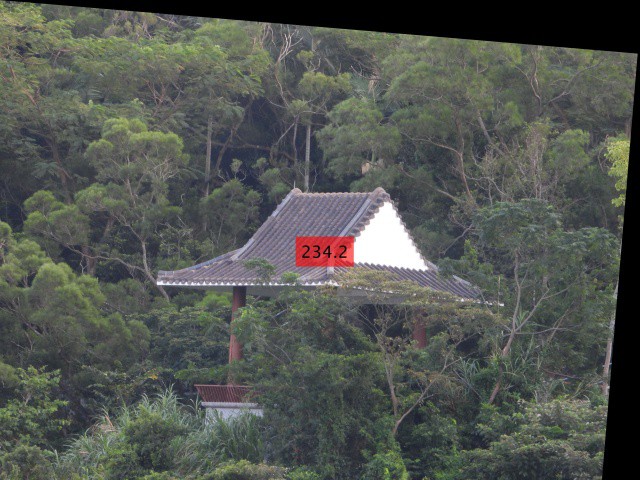}} & 
        \includegraphics[width=0.188\textwidth]{{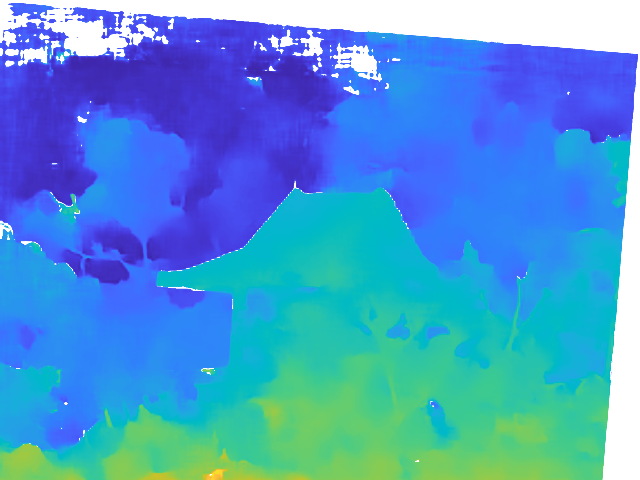}} & 
         \includegraphics[width=0.188\textwidth]{{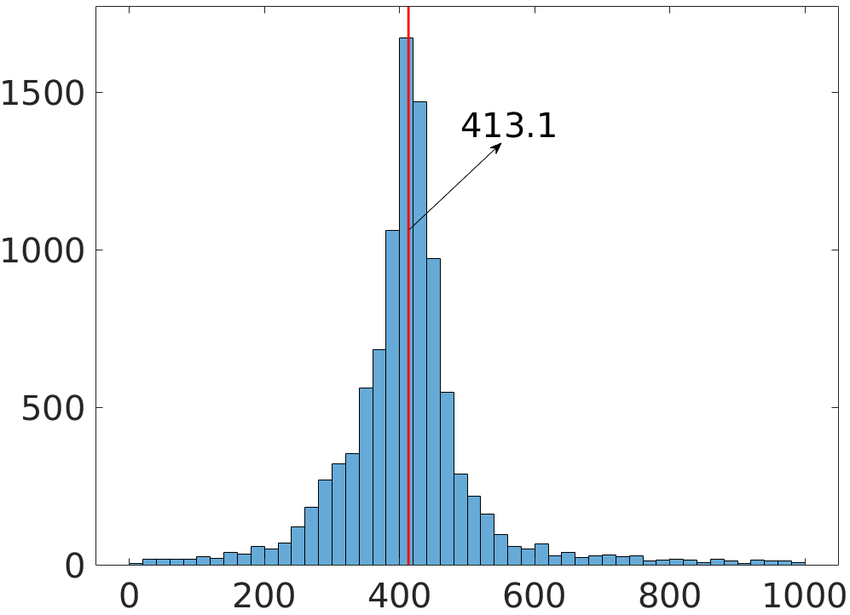}} & 
         \includegraphics[width=0.188\textwidth]{{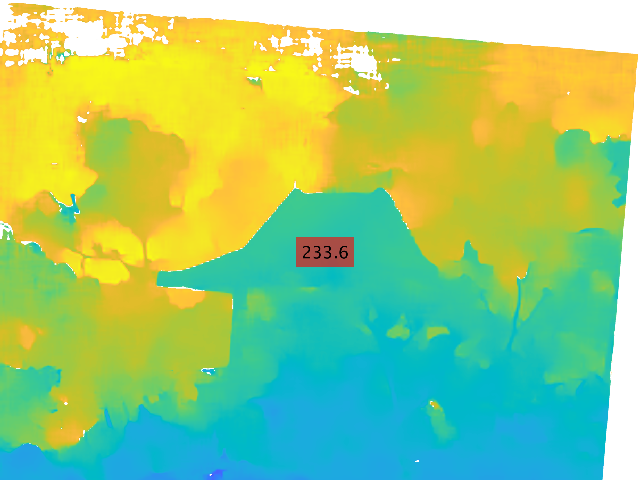}} \vspace{-0.4em}\\
         &     \includegraphics[width=0.188\textwidth]{{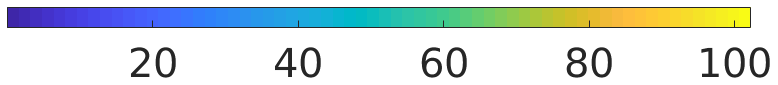}} &   &  \includegraphics[width=0.188\textwidth]{{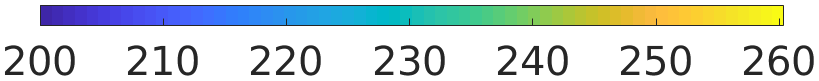}}\\
        \includegraphics[width=0.188\textwidth]{{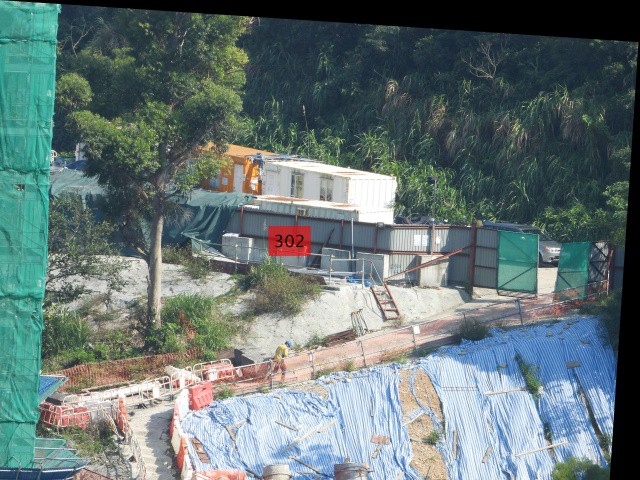}} & 
        \includegraphics[width=0.188\textwidth]{{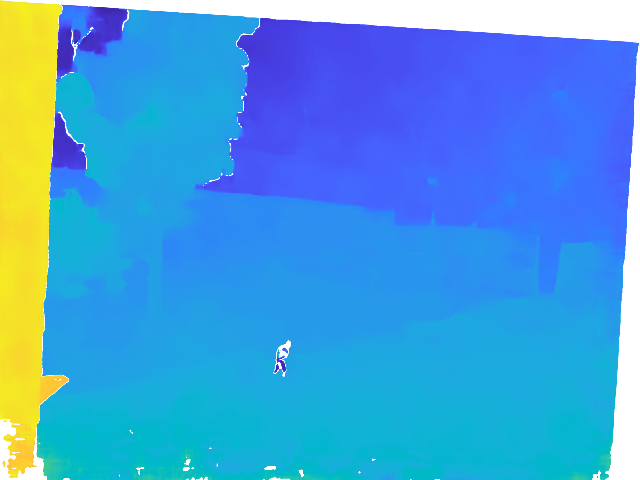}} & 
         \includegraphics[width=0.188\textwidth]{{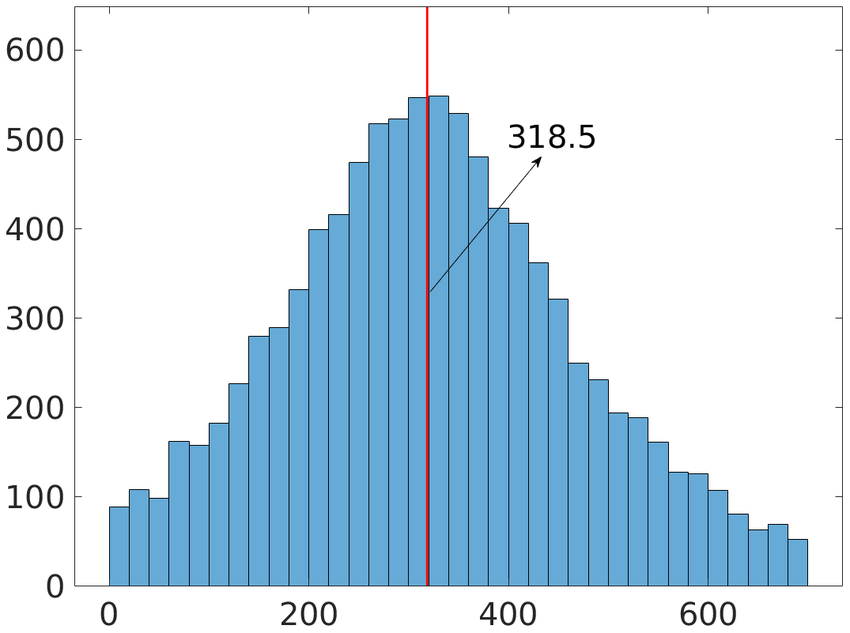}} & 
         \includegraphics[width=0.188\textwidth]{{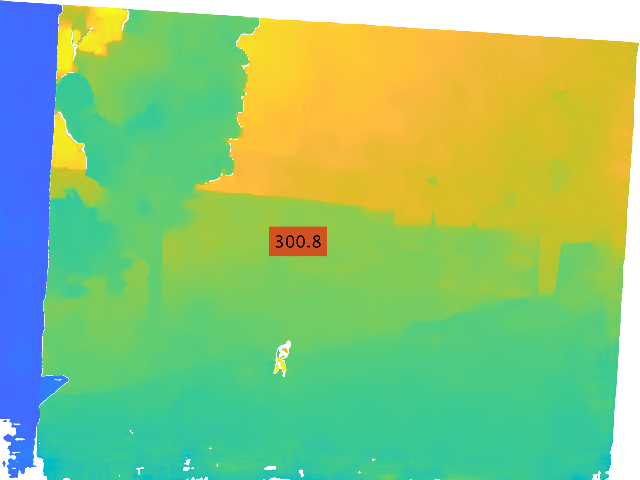}} \vspace{-0.4em}\\
         &     \includegraphics[width=0.188\textwidth]{{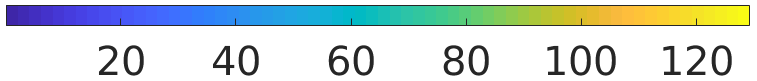}} &   &  \includegraphics[width=0.188\textwidth]{{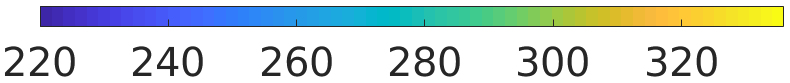}}\\
        \includegraphics[width=0.188\textwidth]{{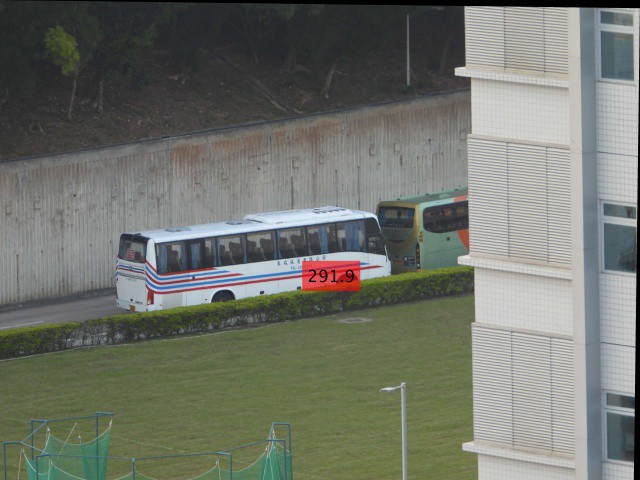}} & 
        \includegraphics[width=0.188\textwidth]{{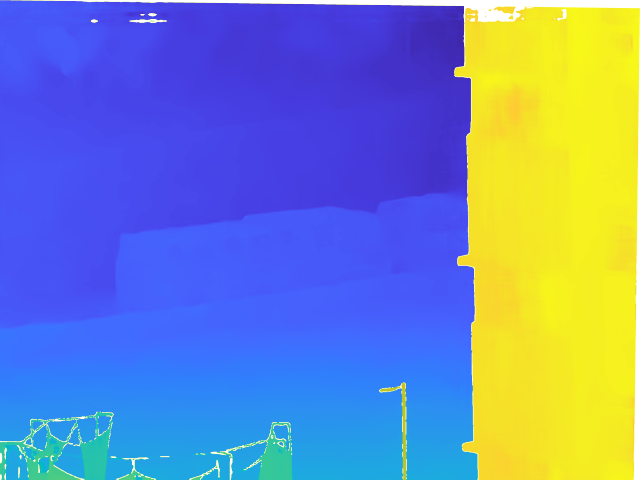}} & 
         \includegraphics[width=0.188\textwidth]{{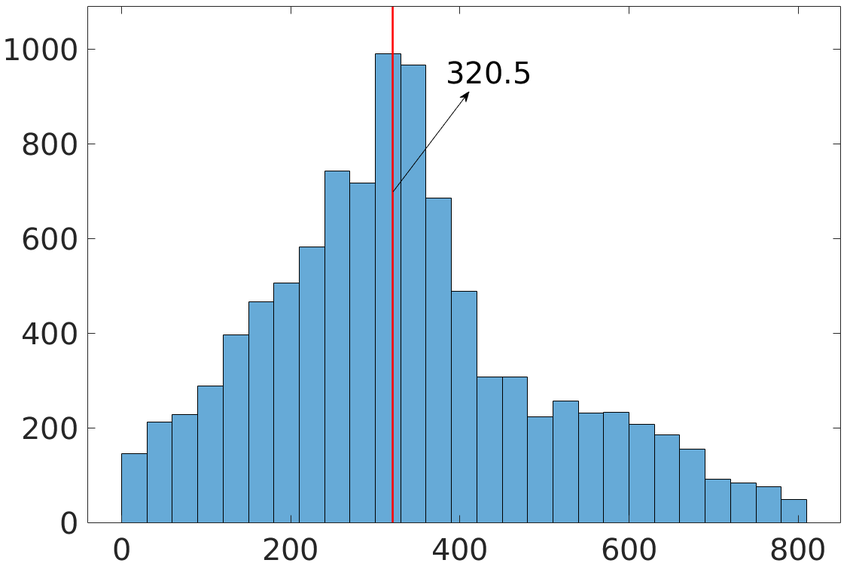}} & 
         \includegraphics[width=0.188\textwidth]{{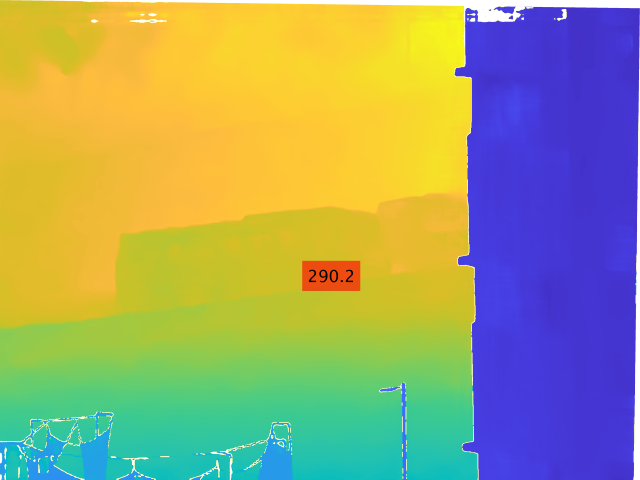}} \vspace{-0.4em} \\
         &     \includegraphics[width=0.188\textwidth]{{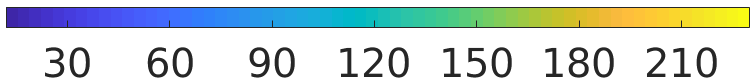}} &   &  \includegraphics[width=0.188\textwidth]{{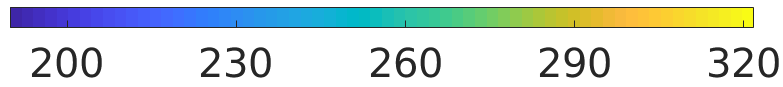}}\\
     \end{tabular} 
     \vspace{-1mm}
  \caption{\small Results on real-world data. The depth measurements from a laser rangefinder are marked in the red box on the pseudo-rectified left images; the corresponding estimated values by our method are marked on the estimated depth maps. Laser-measured values (top to bottom): 234.2m, 302m, 291.9m; our estimated values (top to bottom): 233.6m, 300.8m, 290.2m. 
  }\label{fig:real}
  \vspace{-2mm}
\end{figure*}

\subsection{Real-world data}
We capture real-world data with a Nikon P1000 super-zoom camera; the camera is mounted on a tripod and manually moved to three positions in line with our proposed camera setup in order to acquire the left, right, and back images. The captured images are of the same size, 4608 $\times$ 3456, as in the synthetic case. The 35mm equivalent focal length is 400mm, which corresponds to a camera FOV of $5.16^\circ$ horizontally and $3.44^\circ$ vertically. Because ground-truth dense depth maps are difficult to obtain for distant real-world scenes without special equipment, we use a laser rangefinder to acquire a point-wise depth measurement for a point of interest in the captured scene. We can then check if the estimated depth agrees with the measured one. The hyper-parameters, i.e., $M, \epsilon, \delta,\eta$, remain unchanged compared to the synthetic case.  

We first show qualitative results of pseudo-rectification on real-world images in Fig.~\ref{fig:real_pseudo_rectify}; like before, a 60 $\times$ 45 sub-area cropped out of the rectified views is presented to ease visual inspection. One can see that our pseudo-rectification generalizes very well to real-world images. In Fig.~\ref{fig:real}, we show the estimated depth maps; the measured depth from the laser rangefinder and the corresponding estimated value are marked inside the red boxes on the pseudo-rectified left view and the estimated depth map, respectively. From a practical perspective, the accuracy of our estimated depths is quite acceptable for applications in autonomous driving, considering the large distances to the scenes.

%% file: 06_conclusion.tex
\section{Discussion}
In this work, we propose a novel vision-based solution to the long-range depth sensing problem in autonomous driving. We propose a three-camera system consisting of small-FOV cameras and a corresponding processing pipeline. Our end-to-end solution is very practical in that it does not assume full calibration of the camera system, and is robust to small system vibrations. Experiments show that our system enables \textit{dense} depth acquisition of faraway objects ($>$200m) that are beyond the range of most commercial LiDARs for self-driving vehicles. This can be particularly helpful for heavily-weighted autonomous trucks moving at high speed. 

As future work, we plan to conduct thorough experiments in real-world driving scenarios 
by building and testing a road-deployable hardware system.